\documentclass[dvipsnames,pmlr,twoside,tablecaptionbottom]{pgm-jmlr}
\usepackage{pgm-pmlr-dat}
\firstpageno{1}

\usepackage[arxiv]{optional}
\usepackage[utf8]{inputenc}
\inputencoding{utf8}

\usepackage{booktabs}

\usepackage{xcolor}
\usepackage{tikz}
\usetikzlibrary{arrows.meta,decorations.pathmorphing,decorations.markings,shapes,shapes.geometric}
\pgfdeclarelayer{edgelayer}
\pgfdeclarelayer{nodelayer}
\pgfsetlayers{edgelayer,nodelayer,main}

\newtheorem{prop}[theorem]{Proposition}

\makeatletter
\def\lstAZ{A, B, C, D, E, F, G, H, I, J, K, L, M, N, O, P, Q, R, S, T, U, V, W, X, Y, Z}
\def\lstaz{a, b, c, d, e, f, g, h, i, j, k, l, m, n, o, p, q, r, s, t, u, v, w, x, y, z}

\def\lstAZBB{B, C, D, E, F, G, H, I, J, K, L, M, N, O, P, Q, R, T, U, V, W, X, Y, Z}

\newcommand{\MkScr}[1]{\expandafter\def\csname s#1\endcsname{\mathscr{#1}}}
\newcommand{\MkUp}[1]{\expandafter\def\csname u#1\endcsname{\mathrm{#1}}}
\newcommand{\MkFrak}[1]{\expandafter\def\csname f#1\endcsname{\mathfrak{#1}}}
\newcommand{\MkCal}[1]{\expandafter\def\csname c#1\endcsname{\mathcal{#1}}}
\newcommand{\MkBB}[1]{\expandafter\def\csname #1#1\endcsname{\mathbb{#1}}}

\@for\i:=\lstAZ\do{%
	\expandafter\MkScr \i  %
	\expandafter\MkFrak \i  %
	\expandafter\MkUp \i %
	\expandafter\MkCal \i  %
		  }    
\@for\i:=\lstaz\do{%
	\expandafter\MkUp \i   }    
	
\@for\i:=\lstAZBB\do{%
	\expandafter\MkBB \i     }
	    
\makeatother

\renewcommand{\P}{\mathrm{P}}
\newcommand{\E}{\mathrm{E}}

\newcommand{\on}[1]{{\operatorname{#1}}}

\newcommand{\Pa}{\operatorname{Pa}} 
 
\newcommand{\Ne}{\operatorname{Ne}}

\newcommand{\rauf}{{\scriptstyle +1}}
\newcommand{\runter}{{\scriptstyle -1}}
\newcommand{\Insert}{\operatorname{Insert}}
\newcommand{\Delete}{\operatorname{Delete}}

\newcommand{\NA}{\operatorname{NA}}

\newcommand{\A}{a}
\newcommand{\B}{b}
\newcommand{\States}{S}
\newcommand{\invo}{\mathfrak{s}}

\tikzstyle{none}=[]
\tikzstyle{empty}=[fill=BlueGreen!25, draw=black, shape=circle]
\tikzstyle{solid}=[fill=black, draw=black, shape=circle]

\tikzstyle{doubleedge}=[-{Stealth[length=2mm, width=1.5mm]}, style=double]
\tikzstyle{edge}=[-{Stealth[length=2mm, width=1.5mm]}]
\tikzstyle{punkte}=[{Stealth[length=2mm, width=1.5mm]}-,dotted]

\title[Causal Structure Learning With Momentum]{Causal Structure Learning With Momentum: Sampling Distributions
Over Markov Equivalence Classes}
 \author{
 \Name{Moritz Schauer} \Email{smoritz@chalmers.se} \\
 \addr Chalmers University of Technology and University of Gothenburg
 \AND
 \Name{Marcel Wienöbst}  \Email{m.wienoebst@uni-luebeck.de} \\ 
 \addr Institute for Theoretical Computer Science, University of Lübeck
 }

\begin{document}
\maketitle

\begin{abstract}
In the context of inferring a Bayesian network structure (directed acyclic graph, DAG for short), we devise a non-reversible continuous time Markov chain, the ``Causal Zig-Zag sampler'', that targets a probability distribution over classes of observationally equivalent (Markov equivalent) DAGs. The classes are represented as completed partially directed acyclic graphs (CPDAGs). The non-reversible Markov chain relies on the operators used in Chickering's Greedy Equivalence Search (GES) and is endowed with a momentum variable, which improves mixing significantly as we show empirically. The possible target distributions include posterior distributions based on a prior over DAGs and a Markov equivalent likelihood. 
We offer an efficient implementation wherein we develop new algorithms for listing, counting, uniformly sampling, and applying possible moves of the GES operators, all of which significantly improve upon the state-of-the-art run-time.
\end{abstract}
\begin{keywords}
MCMC; Causal Discovery; Markov Equivalence Classes; DAGs.
\end{keywords}

\section{Introduction}
A Bayesian network is a probabilistic graphical model that represents a set of random variables and their conditional (in)dependencies using a directed acyclic graph (DAG).
Graph and random variables are linked by 
the local Markov condition: variables are conditionally independent of
their non-descendants given their parents, which induces a
factorisation of the joint distribution via conditional distributions
of variables given their parents~\citep{lauritzen1996graphical,koller2009probabilistic}. Typically, there are multiple such factorisations or multiple DAGs such that the local Markov condition holds.

Causal Bayesian networks, in which the edges in the DAG represent direct causal influences, provide a theory of how interventions change the joint distribution of latent and observable variables~\citep{pearl2009causality,peters2017elements}. Here, one assumes the \emph{causal} Markov condition that a variable conditional on its direct causes is independent of variables that are not directly or indirectly influenced by it.
\begin{figure}[h!]
\centering
\includegraphics[width=0.85\linewidth]{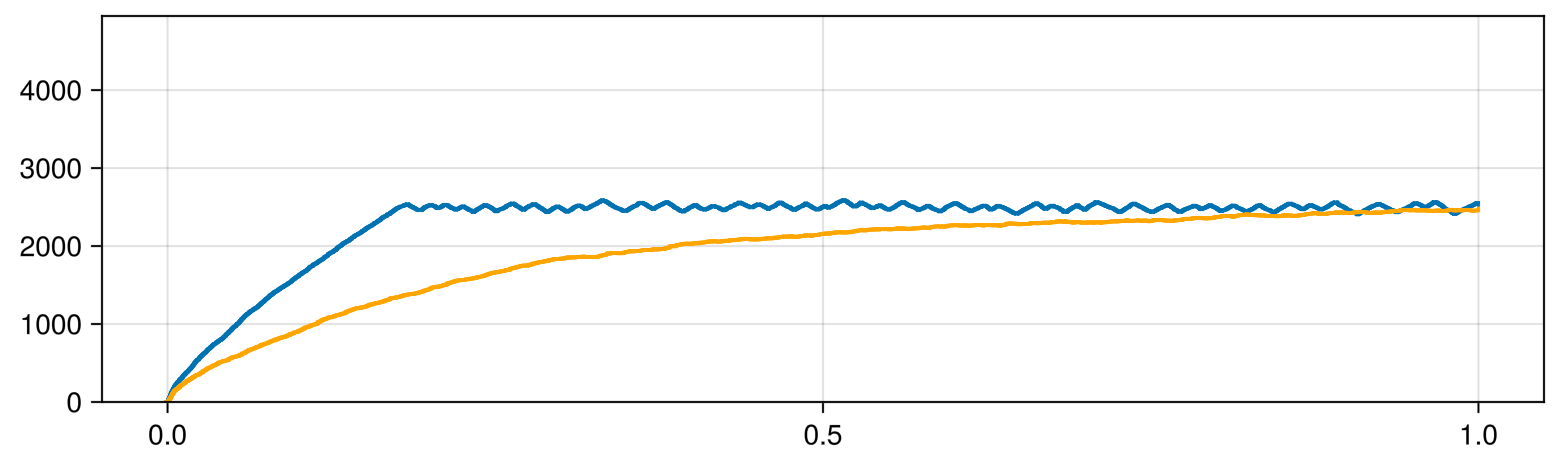}
\caption{Continuous-time trace of the number of edges of the sampled graphs when targeting a uniform distribution on CPDAGs with 100 vertices for the non-reversible sampler proposed here in blue compared with the similar, but reversible, Zanella sampler in orange. The total time of 1 unit corresponds to 5\,000 jumps. Our sampler reaches equilibrium considerably faster and mixes better.
}
\label{fig:convergence100}
\end{figure}
Therefore, even when assuming faithfulness, that all conditional independencies in the data are implied by the factorisation of the underlying causal DAG, observational data is generally insufficient to uniquely determine this graph. Instead the DAGs which cannot be told apart by observational data form a Markov equivalence class (MEC), that is an equivalence class of DAGs~\citep{verma1990equivalence,heckerman1995learning}, usually represented by a \emph{completed partially directed graph} (CPDAG).\footnote{Technical definitions are given in section~\ref{section:preliminaries}.} In Bayesian inference this manifests in marginal likelihoods that are the same for all members of the MEC. Bayesian inference starting from a prior distribution on the equivalence classes hence yields a posterior distribution over Markov equivalence classes.

Markovian Monte Carlo methods allow drawing samples from that posterior distribution. They work by constructing a stochastic process $Z$ with a temporal Markov property,\footnote{Not to be confused with the local Markov condition and the causal Markov assumption in the Bayesian network.}  that has the desired distribution as its equilibrium distribution, see \citet{Roberts2004} for a general account for Markov chains.
The empirical distribution of samples taken from the process then approximates the usually intractable posterior distribution.
This is classically done with discrete time Markov chains, but recently continuous time samplers have also become an active research area \citep{fearnhead2018piecewise}.

In this work, the sampler is based on a stochastic process $Z = (Z_t)_{t \ge 0}$ taking values $(\gamma, d)$ in a space of extended coordinates where $\gamma \in \cM_n$ is a MEC on $n$ variables and $d \in \{\runter, \rauf\}$ is variable indicating a direction of movement corresponding to adding edges to $\gamma$ if $d = \rauf$ and removing edges from $\gamma$ if $d = \runter$. 
This is analogous to \citet{GUSTAFSON1998} who adds a direction variable to a random walk on the integers in order to improve mixing by allowing for repeated moves in the same direction in contrast to choosing a random direction in every step. Also Hamiltonian Monte Carlo uses a momentum variable to balance random walk behaviour and systematic exploration \citep{Neal1996}.

The sampler relies on the operators introduced by \citet{chickering2002optimal}  in the celebrated GES algorithm for estimating a single MEC. They allow to move between MECs, which have DAG members that differ only by a single edge deletion or insertion, thus providing a natural and efficient representation of this space. 
Moreover, they can be used to immediately obtain a reversible Markov chain, as for example recently explored by~\citet{zhou2023complexity} which propose a locally balanced Markov chain sampler in the sense of \citet{Zanella2019} for the problem.
Endowing them with momentum improves mixing and retains closeness to the GES approach, where there are two main phases: (i) the forward phase, during which edges are inserted and (ii) the backward phase, during which edges are deleted. Indeed, our algorithm, which we term \emph{Causal Zig-Zag}, can be viewed as a generalisation of GES and it converges to it in the limit of increasing coldness given by a thermodynamic $\beta$ as we show in section~\ref{sec:ges}. Because GES itself provably recovers the MEC of the underlying true DAG in the limit of large sample size, this translates to Causal Zig-Zag, which is effective in finding high-posterior regions. 
More generally, we make the following contributions.
\begin{enumerate}
    \item We present a sampler for Markov equivalence classes that is both non-reversible and locally balanced with application to Bayesian causal discovery and causal discovery with uncertainty quantification. Similar to the GES algorithm the sampler operates in alternating phases, one phase where edges are inserted and one phase where edges are removed. This makes the sampler non-reversible and improves mixing.

     \item We base the sampler on new, efficient algorithms for listing, counting and applying possible moves in the space of MECs based on Chickering's Insert and Delete operators. These improvements go beyond the use cases in this work and also apply to the original GES and related algorithms. 
    \item We show the benefits and practicality of our approach empirically and make our implementation available in the software package \href{https://github.com/mschauer/CausalInference.jl}{CausalInference.jl}.
\opt{submission}{
}
\end{enumerate}

As first illustration, we use our non-reversible sampler and a reversible counterpart to sample CPDAGs with 100 vertices uniformly. Both samplers start from the empty graph and continue for 5\,000 steps. The samplers require no further choice of tuning parameters. Our sampler reaches equilibrium considerably faster, see figure~\ref{fig:convergence100}. The time of reaching a large set such as, in this case, CPDAGs with $2400$ to $2600$ edges, from a single state (the empty graph) is informative about mixing times~\citep{Peres2013}.

\begin{figure*}
\centering
\includegraphics[width=0.45\linewidth]{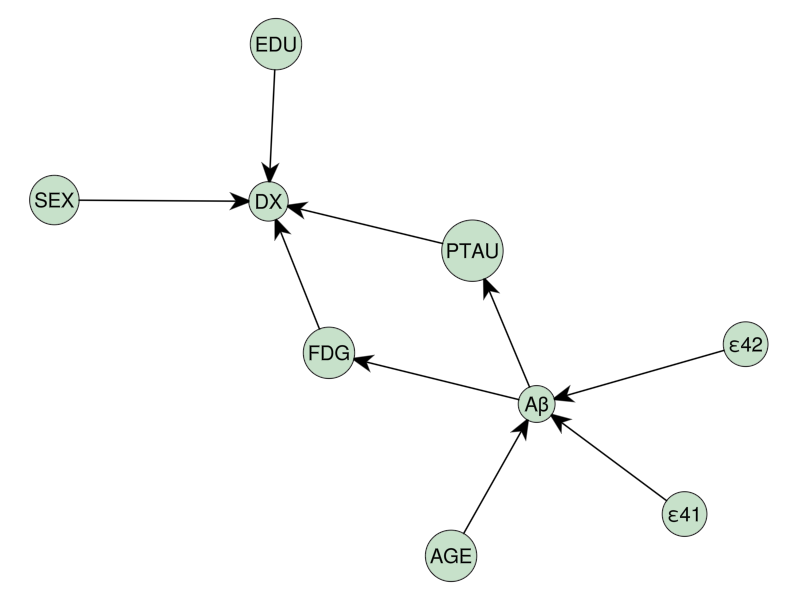}
\hspace{0.5cm}
\includegraphics[width=0.45\linewidth]{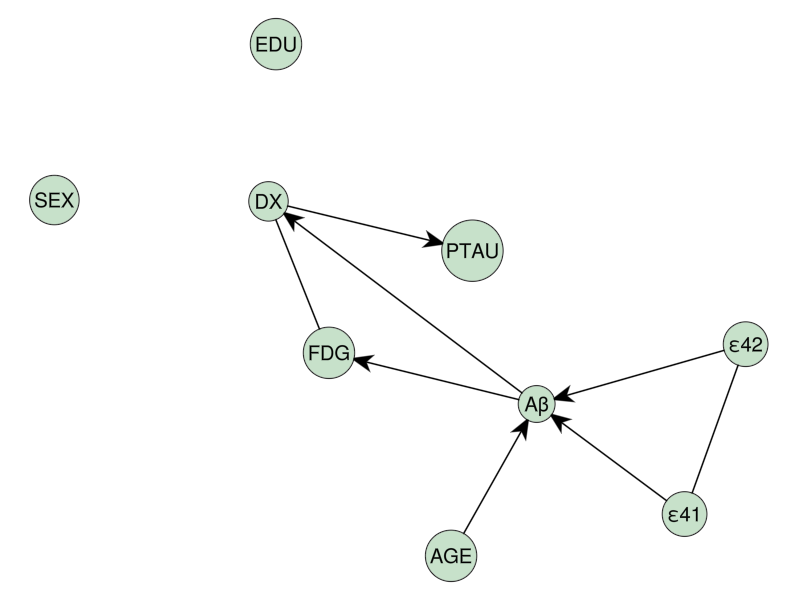}

\caption{The expert assessment of the causal model on the left. On the right, the model with highest posterior probability 0.701, which coincides with the model found by GES. The two models with next highest posterior probabilities are shown in Figure~\ref{fig:adniappendix} in the Appendix. }
\label{fig:adni}
\end{figure*}

As second illustration, we partly reproduce \citet{Shen2020}. They consider data from the Alzheimer’s Disease Neuroimaging
Initiative (ADNI) database (\url{adni.loni.usc.edu}).\footnote{See acknowledgements for more information.}
The variables extracted from the data are fludeoxyglucose PET (FDG), amyloid beta (A$\beta$),
phosphorylated tau (PTAU), number of $\varepsilon 4$ alleles of apolipoprotein E; demographic information: age, sex, years of education
(EDU); and diagnosis on Alzheimer disease (DX).
To account for possibly non-linear effects the number of $\varepsilon4$ alleles (0, 1, or 2) is dummy encoded ($\varepsilon42$, $\varepsilon41$), as it is done in \citet{Shen2020}.
We use our algorithm to sample CPDAGs proportional to their (exponentiated) BIC score with penalty $5.5$ and run the sampler for 50\,000 jumps starting from the empty graph. See section 4.1 in \citet{chickering2002optimal} for a discussion of the Bayesian Information Criterion (BIC) and its relationship to the marginal posterior. Our findings are shown in figures \ref{fig:adni} and \ref{fig:adni-levels}.

\begin{figure}
\centering
\includegraphics[width=0.85\linewidth]{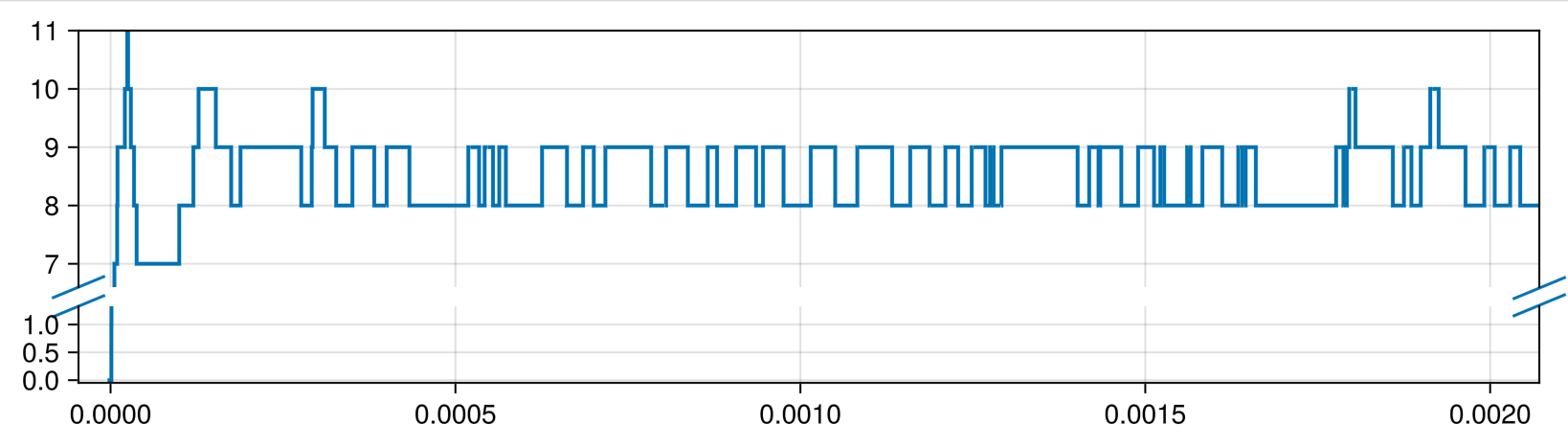}
\caption{Continuous-time trace of the number of edges of the first sampled graphs for the ADNI data. At this time scale, the random time spend in each CPDAG is visible.}
\label{fig:adni-levels}
\end{figure}

\section{Related work} 
 
Bayesian methods for learning DAGs from observational data, which directly target the posterior probability over MECs, as we do in this work, are underrepresented in the literature with popular exact methods estimating the marginal posterior probability of every possible edge \citep{koivisto2004exact} and MCMC samplers focusing on the space of DAGs~\citep{madigan1995bayesian,giudici2003improving,grzegorczyk2008improving} or variable orderings~\citep{friedman2000being,niinimaki2016structure,kuipers2017partition,agrawal2018minimal} being more widespread. Recently, differentiable formulations  
have been pursued and exploited by variational and MCMC methods~\citep{lorch2021dibs,annadani2021variational,cundy2021bcd,deleu2022bayesian,annadani2023bayesdag}. 

On the other hand, when aiming to estimate a single causal structure, classical algorithms such as PC~\citep{spirtes2000causation} and GES~\citep{chickering2002optimal} 
are at their core build on the notion of Markov equivalence.  More generally, exploiting as well as analysing the space and properties of MECs has a long and fruitful history in the causal discovery and Bayesian network communities, beginning with \citet{madigan1996bayesian}, who pivoted the use of MCMC using the search space of MECs for Bayesian structure learning, and \citet{gillispie2002size}, who initiated studies of the size distribution of MECs. Later these works were extended by \citet{pena2007approximate,he2013reversible}, who again used MCMC
to analyse, e.g, the average number of undirected edges in a CPDAG, focusing mainly on sparse graphs. 
Recently, \citet{zhou2023complexity} showed that the GES operators by~\citet{chickering2002learning} have superior mixing properties compared to these earlier MCMC approaches.
The sampler used by \citet{zhou2023complexity} belongs to a class of discrete time locally balanced sampler in high dimensional spaces \citep{Zanella2019}.
For the continuous time perspective,
see \citep{https://doi.org/10.48550/arxiv.1912.04681}.

\section{Preliminaries} \label{section:preliminaries}

\paragraph{Graphs and notation.}
A partially directed graph, here short ``graph'', $G = (V, E)$ consists of a set of $n$ vertices $V$ and a set of $m$
edges $E \subseteq V \times V$.\footnote{Excluding self-edges: $(x,x) \notin E$.}
An undirected edge between vertices $x, y \in V$, denoted $x-y$, has both  
$(x, y) \in E$ and $(y, x) \in E$, and a directed edge $u \rightarrow v$ has $(x, y) \in E$
and $(y, x) \not\in E$.  Vertices linked by an edge (of any type) are
\emph{adjacent}, and vertices linked by undirected edges are \emph{neighbours} of each other. We say that $x$ is
a \emph{parent} of $y$ if $x\rightarrow y$.  We denote by $\Pa(x)$
and $\Ne(x)$ the set of parents and neighbors of~$x$. A directed graph contains no undirected edges.
A partially directed acyclic graph (PDAG) is a graph without directed cycles and a directed acyclic graph (DAG) is a \emph{directed} graph with this property.
We denote the space of DAGs over $n$ vertices as $\cD_n$.
We let $\uU(S)$ denote the uniform distribution on a set $S$. $\sqcup$ denotes the disjoint union of sets.

\paragraph{Markov equivalence classes.}
In case of a Bayesian network, the vertex set $V$ is a set of random variables. 
A \emph{v-structure} are vertices $x, y, z$ such that $x \rightarrow y \leftarrow z$ and $x, z$ are not adjacent. All DAGs on a vertex set $V$ with the same set of v-structures and the same set of adjacencies are observationally equivalent or \emph{Markov-equivalent} as shown by~\citet{verma1990equivalence} and form the Markov equivalence class (MEC).  
A CPDAG (completed PDAG) has $x \rightarrow y$, if $x \rightarrow y$ in each member of the equivalence class, and $x - y$, if there are DAGs $G$ and $G'$ in the MEC such that $G$ contains $x \rightarrow y$ and $G'$ contains $x \leftarrow y$. The CPDAG uniquely determines the MEC. We denote the space of CPDAGs or MECs  as $\cM_n$ and denote its elements by $\gamma, \eta, \dots \in \cM_n$. 
A scoring function $\cD_n \to [0, \infty)$ for DAGs is a \emph{Markov equivalent score} if it assigns the same score to any DAG in the same MEC.

\paragraph{Markov jump process.}
Following \citet{kallenberg2002foundations}, 
a continuous time stochastic process $(Z_t)_{t \ge 0}$ on a countable state space $\States$ with almost surely right-continuous paths that are constant apart from isolated jumps with the temporal Markov property is a Markov jump process.\footnote{We only consider time-homogeneous processes where
$
\P(Z_{t} = \B \mid Z_{s} = \A)$
only depends on $t-s$.}

In our case, the state space is the space of MECs $\States = \cM_n$ or the space of MECs extended by a direction or momentum, $\States = \cM_n \times \{\rauf, \runter\}$, and an abstract notion of time inherent to the sampler, related but not identical to the run time of its implementation.

Denote the jump times of $Z$ as $0 = \tau_0 < \tau_1 <\tau_2 < \dots$, these are random times $\tau$ where $Z_{\tau} \ne Z_{\tau-}$.
The law of a Markov jump process can be described by 
\begin{itemize}
    \item the starting distribution $Z_0 \sim \nu$; 
\item the rate function $\Lambda\colon \States \to [0, \infty)$
such that conditional on $Z_{\tau_i} = \A$, $\A \in \States$, the time to the next jump $\tau_{i+1} - \tau_i$ is exponentially distributed with rate $\Lambda$ depending on $\A$;\footnote{So $\Lambda(\A) = 1/(\E[\tau_{i+1} - \tau_i \mid Z_{\tau_i} = \A])$.} 
\item a jump kernel, such that
$Z_{\tau_{i+1}}$ has the conditional distribution $\kappa_\A$ given $Z_{\tau_{i}} = \A$.
\end{itemize}

This entails by the Markov property that $\tau_1/\Lambda(Z_0)$, $(\tau_{2} - \tau_1)/\Lambda(Z_{\tau_1})$, $\dots$ form an independent sequence of $\operatorname{Exp}(1)$ random variables and
$Z_0, Z_{\tau_1}, Z_{\tau_2}, \dots$ an embedded discrete-time Markov chain where $P(Z_{\tau_i} = \B \mid Z_{\tau_i} = \A) = \kappa_\A \{\B\}$ with $\kappa_\A$ for $\A \in \States$ being a probability kernel
$\sum_{\B \in \States} \kappa_\A\{\B\} = 1$
where $\kappa_\A\{\A\} = 0$ by construction.

We also define $\lambda(\A \curvearrowright \B) = \Lambda(\A)\kappa_\A  \{\B\}$ the specific rate of jumps from $\A \in \States$ to $\B \in \States$. Both total rate $\Lambda(a)$ and the jump kernel $\kappa_a$, $a \in S$,  are determined by $\lambda$ through
$\Lambda(\A) = \sum_{\B \in \States} \lambda (\A \curvearrowright \B)$ 
and
$
\kappa_\A\{\B\} = \frac{\lambda(\A\curvearrowright\B)}{\Lambda(\A)},$, $b \in S.$
This has intuitive meaning. As the minimum of independent exponential random variables with rates $\lambda(a \curvearrowright b_1)$, \dots, $\lambda(a \curvearrowright b_k)$ is exponentially distributed with rate $\Lambda(a)$, one can either jump to a state drawn from $\kappa_a$ after $\on{Exp}(\Lambda(a))$ distributed time units, or chose the earliest jump to $b_1$, \dots, $b_k$ in the support of $\kappa_a$ with jump times drawn each from  (independent) distributions $\on{Exp}(\lambda(a\curvearrowright b_1))$, \dots, $\on{Exp}(\lambda(a\curvearrowright b_1))$.

A process has $\pi$ as equilibrium distribution if
$
\sum_{\A \in \States} \P(Z_t \in B \mid Z_s = \A) \pi\{ \A\} = \pi(B),
$
where $ t > s > 0, B \subset \States.$
A stronger requirement relevant for sampling is ergodicity, which for finite state spaces takes the form
$
\lim_{t \to \infty} \P\left(Z_t = b \mid Z_s = a\right) = \pi\{b\} 
$
for all $b, a \in S$ so that in the long run, states from $Z$ can be used to approximate samples from $\pi$.

\paragraph{Operators for Markov equivalence classes. }

\citet{chickering2002learning} defines two sets of operators on $\cM_n$.
The operator $\Insert(\gamma, x, y, T)$
inserts the edge $x \to y$ to the CPDAG $\gamma$ and directs previously undirected edges $t-y$ to $t\to y$ for $t \in T$, such that vertices $t \in T$ become ``tails'' of a v-structure $t \rightarrow y\leftarrow x$. Here $x$ and $y$ are not adjacent and $T$ are (undirected) neighbours of $y$ that are not adjacent to $x$. The resulting PDAG is then completed\footnote{The \emph{completion} of a PDAG refers to the CPDAG representation of the MEC with the same skeleton and v-structures as the PDAG. There are cases, when this CPDAG does not exist, namely when there are no DAGs with this skeleton and v-structures. A simple example is PDAG $C_4$, the cycle on four vertices.} to a CPDAG $\gamma'$ if possible, otherwise the insertion is not defined (invalid).

The operator $\Delete(\gamma', x, y, H)$
deletes an edge  $x-y$ or $x\to y$ of the CPDAG $\gamma'$ and directs previously undirected edges $x-h$ as $x\to h$ and $y - h$ as $y\to h$ for $h$ in $H$ such that vertices $h \in H$ become ``heads'' of new v-structures $x \rightarrow h \leftarrow y$. The resulting PDAG is then completed to a CPDAG $\gamma$ if possible, otherwise the deletion is not defined (invalid).

We call a move or jump from MEC $\gamma$ to MEC $\gamma'$ \emph{local} if there is a DAG $G \in \gamma$, which can be transformed to a DAG $G' \in \gamma'$ by a single edge insertion or deletion. Local moves are preferable for two reasons: Firstly, if a weight function $w$, for example the exponentiated BIC score, factorises over the DAGs, 
\[
w(G, \operatorname{Data}) = \prod_{x \in V} w(\Pa_G(x), x, \operatorname{Data}),
\]
then changes in $w$ can be computed efficiently by comparing local scores or local weights, see \citet{chickering2002learning}, corollaries 7 and 9.

Secondly, Theorems~15 and~17 of \citep{chickering2002learning} give precise criteria for the validity of local moves.
Denote by $\NA_x(y)$ the (undirected) neighbours of $y$ that are adjacent to $x$.
 In short, $\Insert(\gamma, x, y, T)$ is a valid local move, if and only if
(i) $\NA_{x}(y)$ and the elements of $T$ form a clique and
(ii) any path from $y$ to $x$ without a directed edge pointing towards $y$ (such a path is called semi-directed) contains a 
vertex in $\NA_x(y) \cup T$.
$\Delete(\gamma, x, y, H)$ is a valid local move, if and only if $H\subset \NA_x(y)$ and $\NA_{x}(y)\setminus H$ form a clique.

\section{Random walks on Markov equivalence classes}

The key for the construction of a Markov process on Markov equivalence classes is that the valid local $\Insert$ and $\Delete$ operators are mutual inverses.

\begin{lemma}[\cite{chickering2002optimal,zhou2023complexity} ]
If $\gamma' = \on{Insert}(\gamma, x, y, T)$, $x, y \in V$, $T \subset V$, $\gamma \in \cM_n$ is a valid local move, then there is a unique set of undirected neighbours $H$ of $y$ that are adjacent to $x$ in  $\gamma'$ such that
$
\gamma = \on{Delete}(\gamma', x, y, H).
$

Conversely if $\gamma = \on{Delete}(\gamma', x, y, H)$ is a valid local move, then there is a unique set of undirected neighbours $T$ of $y$ that are not adjacent to $x$ in $\gamma$ such that
$
\gamma' = \on{Insert}(\gamma, x, y, T).
$
\end{lemma}

There may be two operators going from $\gamma$ to $\gamma'$, which is precisely the case if the inserted or deleted edge is undirected and $\on{Insert}(\gamma, x, y, T)$ equals $\on{Insert}(\gamma, y, x, T)$ (same for $\on{Delete}$). Phrased differently, the number of operators turning $\gamma$ into $\gamma'$ is identical to the operators for the reverse direction from $\gamma'$ to $\gamma$~\citep{zhou2023complexity}. 

\begin{lemma}[\cite{chickering2002optimal,zhou2023complexity}]\label{lemma:undirected}
The edge inserted by a local $\on{Insert}(\gamma, x, y, T)$ is undirected exactly if 
$T$ is empty and $\Pa(x) = \Pa(y)$. 
\end{lemma}

We write $\gamma' \in \Insert(\gamma)$ and $\gamma \in \Delete(\gamma')$ to indicate that $\gamma'$  can be obtained from $\gamma$ by a valid \emph{local} $\Insert$ operation and that $\gamma$ can be obtained from $\gamma'$ by a valid \emph{local} $\Delete$ operation. 
For example this lemma entails, when declaring $\gamma, \eta \in \cM_n$ (undirected) neighbours if $\eta \in \Insert(\gamma)\cup \Delete(\gamma)$,
general algorithms to sample from undirected graphs such as a simple continuous time random walk on $S = \cM_n$ with jump intensity
\[
\lambda(\gamma \curvearrowright \eta) = \begin{cases}
    1 &\text{if $\eta \in \Insert(\gamma)\sqcup \Delete(\gamma)$}\\
    0 &\text{otherwise.}
\end{cases}
\]
This process has $\uU(\cM_n)$ as stationary distribution.
While this jump intensity is remarkably simple, practical implementation requires the efficient enumeration of valid $\Insert$ and $\Delete$ operators
for example to determine the total rate $\Lambda(\gamma) = |\Insert(\gamma)\sqcup\Delete(\Gamma)|$, a topic we come back to in  section \ref{section:algorithms}.
Here using lemma \ref{lemma:undirected} allows to account for multiple moves yielding the same CPDAG $\eta$.

Alternatively, one can also move towards $\eta$ with twice the rate if there are two operators from $\gamma$ to $\eta$, as long as one then also moves back from $\eta$ to $\gamma$ with twice the rate. This leads to an easier implementation and thus we proceed this way in our code. 
Also the Zanella process \citep{https://doi.org/10.48550/arxiv.1912.04681}, a generalisation of the simple continuous time random walk that 
can be used to sample from the a distribution $\pi$ defined on $\cM_n$, is now available. 

Let $\pi$ be a probability distribution on $\cM_n$.
Let $g\colon [0,\infty)\to [0, \infty)$ be a balancing function such as $\sqrt{t}$, $\min(1, t)$ or $t/(1+ t)$ with the property $g(t) = tg(1/t)$. The Zanella process $(Z_t)_{t \ge 0}$ on $\cM_n$ is defined by the intensity
\[
\lambda(\gamma \curvearrowright \eta ) =\begin{cases} g\left(\dfrac{\pi\{\eta\}}{\pi\{\gamma\}}\right) & \text{ if $\eta \in \Insert(\gamma)\sqcup \Delete(\gamma)$ } 
\\ 0 &\text{ otherwise}
\end{cases}, 
\]
where $\gamma \in \cM_n$.

\begin{theorem}
\label{zanella}
Let the target probability $\pi$ be strictly positive for all $\gamma \in \cM_n$.  
Then $Z$ is irreducible, $\pi$ is the unique stationary distribution and
\[
\lim_{t \to \infty} \P\left(Z_t = \gamma \mid Z_s = a\right) = \pi\{\gamma\} \quad \text{ for all $\gamma, \eta \in \cM_n$}.
\]
\end{theorem}
The proof of this theorem goes along similar lines as the proof of Theorem \ref{irreducible} below, so we omit it. 

\begin{figure}
    \centering
\begin{tikzpicture}[yscale=0.8]
	\begin{pgfonlayer}{nodelayer}
		\node [style=empty] (2) at (-0.5, 2) {$\eta_1$};
		\node [style=empty] (3) at (0.5, 2) {$\eta_2$};
		\node [style=empty] (8) at (0, 0) {$\gamma$};
		\node [style=empty] (11) at (-1.5, -2) {$\zeta_1$};
		\node [style=empty] (12) at (-0.5, -2) {$\zeta_2$};
		\node [style=empty] (13) at (0.5, -2) {$\zeta_3$};
		\node [style=empty] (14) at (1.5, -2) {$\zeta_4$};
	\end{pgfonlayer}
	\begin{pgfonlayer}{edgelayer}
  		\draw [style=none] (8.center) to (2.center);
		\draw [style=none] (8.center) to (3.center);
		\draw [style=none] (8.center) to (11.center);
		\draw [style=none] (8.center) to (12.center);
		\draw [style=none] (8.center) to (13.center);
		\draw [style=none] (8.center) to (14.center);
	\end{pgfonlayer}
\end{tikzpicture}
\hspace{2.0cm}
\begin{tikzpicture}
	\begin{pgfonlayer}{nodelayer}
		\node [style=empty] (0) at (0, 0) {};
		\node [style=empty] (1) at (3, 0) {$\gamma^\runter$};
		\node [style=none] (2) at (-0.5, 2) {};
		\node [style=none] (3) at (0.5, 2) {};
		\node [style=none] (4) at (2.25, -2) {};
		\node [style=none] (5) at (2.75, -2) {};
		\node [style=none] (6) at (3.25, -2) {};
		\node [style=none] (7) at (3.75, -2) {};
		\node [style=empty] (8) at (0, 0) {$\gamma^\rauf$};
		\node [style=none] (9) at (2.5, 2) {};
		\node [style=none] (10) at (3.5, 2) {};
		\node [style=none] (11) at (-0.75, -2) {};
		\node [style=none] (12) at (-0.25, -2) {};
		\node [style=none] (13) at (0.25, -2) {};
		\node [style=none] (14) at (0.75, -2) {};
	\end{pgfonlayer}
	\begin{pgfonlayer}{edgelayer}
		\draw [style=edge] (8.center) to (2.center);
		\draw [style=edge] (8.center) to (3.center);
		\draw [style=edge, in=165, out=15] (8.center) to (1);
		\draw [style=edge] (1) to (4.center);
		\draw [style=edge] (1) to (5.center);
		\draw [style=edge] (1) to (6.center);
		\draw [style=edge] (1) to (7.center);
		\draw [style=edge, in=-165, out=-15] (8.center) to (1);
		\draw [style=punkte] (1) to (9.center);
		\draw [style=punkte] (1) to (10.center);
		\draw [style=punkte] (8) to (11.center);
		\draw [style=punkte] (8) to (12.center);
		\draw [style=punkte] (8) to (13.center);
		\draw [style=punkte] (8) to (14.center);
	\end{pgfonlayer}
\end{tikzpicture}
        \caption{On the left, MEC $\gamma$ with two neighbours $\eta_1, \eta_2$ in $\Insert(\gamma)$ and four neighbours $\zeta_1, \dots, \zeta_4$ in $\Delete(\gamma)$. The Zanella sampler for the uniform distribution on the space of MECs $\cM_n$ will leave $\gamma$ after an exponentially distributed time with total rate $\Lambda(\gamma) = 6$ towards one of the six neighbours drawn from $\kappa_\gamma = \uU(\{\eta_1, \eta_2, \zeta_1, \zeta_2,\zeta_3, \zeta_4\})
    $. On the right, the situation is shown for the Zig-Zag sampler. To target a uniform distribution on $\cM_n$, if $\gamma \in \cM_n$ has $2$ direct neighbours in $\Insert(\gamma)$  and $4$ direct neighbours in $\Delete(\gamma)$, then move up from $\gamma^\rauf$ with total rate 2, move from $\gamma^\rauf$ to $\gamma^\runter$ with rate $2 = 4 - 2$ and down from $\gamma^\runter$ with total rate $4$.}
    \label{fig:reversible}
\end{figure}
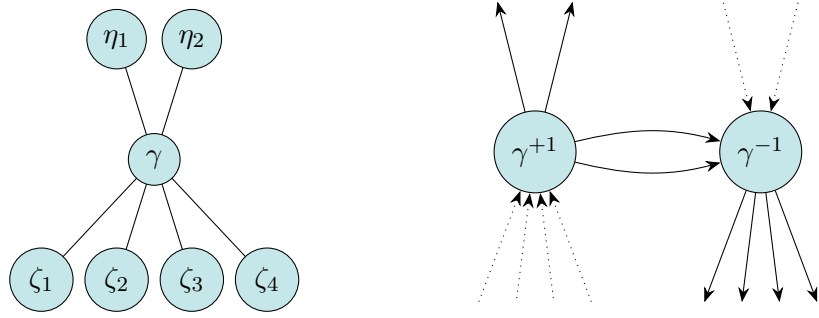

\section{The Causal Zig-Zag sampler}

We now define our sampler which can be thought of as Zanella process lifted by attaching a notion of direction. 
We baptise the non-reversible continuous-time sampler for Markov equivalence classes the ``Causal Zig-Zag'' motivated by the characteristic Zig-Zag pattern in the trace of the number of edges in the causal graph, see figure~\ref{fig:convergence100}.
Here, we exploit that $\Insert$ and $\Delete$ endow the space $\cM_n$ with an intuitive interpretation of direction.

Let $\States = \cM_n \times \{\runter, \rauf\}$. If $\gamma \in \cM_n$, we denote the element $(\gamma, \rauf)$ by $\gamma^\rauf$ and the element $(\gamma, \runter)$ by $\gamma^\runter$ and write $\gamma^d = (\gamma, d)$ for $d \in \{\runter, \rauf\}$.
Again, choose a balancing function $g$ and a target probability $\pi$ on $S$ and a Markov jump process $Z$ as follows:
For $\gamma \in \cM_n$,
\[
\lambda(\gamma^\rauf \curvearrowright \eta^\rauf) =\begin{cases} g\left(\dfrac{\pi\{\eta\}}{\pi\{\gamma\}}\right) & \text{ if $\eta \in \Insert(\gamma)$ }
\\ 0 &\text{ otherwise}.
\end{cases}
\]
\[
\lambda(\eta^\runter \curvearrowright \gamma^\runter) =\begin{cases}  g\left(\dfrac{\pi\{\gamma\}}{\pi\{\eta\}}\right) & \text{ if $\gamma \in \Delete(\eta)$ }
\\ 0 &\text{ otherwise}.
\end{cases}
\]
and for $\gamma \in \cM_n$ and $d \in \{\runter, \rauf\}$,\\
{\small
\hspace*{0.4cm}$
{\lambda( \gamma^ d \curvearrowright \gamma^{ -d}) =  \left(-\sum\limits_{\eta}  \lambda(\gamma^d \curvearrowright \eta^ d) \right.  + \left.\sum\limits_{\eta} \lambda( \gamma^{-d} \curvearrowright \eta^{-d})  \right)^+\!\!\!,}
$}\\
where $x^+ = \max(0,x)$ denotes the positive part. Note that $\lambda$ can be computed if $\pi$ is only known up to a multiplicative constant as typical for Bayesian applications. Figure~\ref{fig:reversible} illustrates the neighboring states for the Zanella and Zig-Zag sampler. 

\begin{theorem}
\label{irreducible}
Let the target probability $\pi\{\gamma\} > 0$ be strictly positive for all $\gamma \in \cM_n$.  
Then $Z$ is irreducible,
\[\P(Z_t = \B \mid Z_s = \A) > 0
\]
for all $\A, \B \in \States$.
The distribution $\tilde \pi$ on $S$ with $\tilde\pi(\gamma^d) = \pi\{\gamma\}/2$ is the unique stationary distribution and
\[
\lim_{t \to \infty} \P\left(Z_t = \gamma^d \mid Z_s = a\right) = \pi\{\gamma\}/2 \quad \text{ for all $ a \in S$}, 
\]
where $ \gamma \in \cM_n, d \in \{\rauf, \runter\}$.
\end{theorem}
\begin{proof}
One first shows that any state $\gamma^d$ communicates with ${\mathbf 0}_n^\runter$, where $\mathbf 0$ denotes the empty graph. From this, the chain $Z$ is irreducible (aperiodicity is not a concern for continuous time chains.) This part of the proof we give in the supplement (it bears some similarity to the consistency argument for the greedy equivalence search algorithm.)

It remains to show that $\tilde \pi$ is the stationary distribution of $Z$. 
This follows by applying proposition \ref{prop:stationary} in the supplement which gives general criteria for stationarity. We proceed by checking the three conditions of the proposition (equations \eqref{isometry}, \eqref{skewbalance} and \eqref{semilocal}).

Firstly, $\invo\colon \States \to \States$, $\invo(\gamma^d) =\gamma^{-d}$ is a bijection on $S$ that is easily seen to be $\tilde\pi$-isometric (equation \eqref{isometry}).\\ 
Also, skew balance (equation \eqref{skewbalance}) holds: if $\eta \in \Insert(\gamma)$
\[
\tilde \pi\{\gamma^{\rauf}\}\lambda(\gamma^{\rauf} \curvearrowright \eta^{\rauf}) =\frac{
\pi\{\gamma\}}2  g\left(\frac{\pi\{\eta\}}{\pi\{\gamma\}}\right)  
\]
\[
= \frac{\pi\{\eta\}}{2} g\left(\frac{\pi\{\gamma\}}{\pi\{\eta\}}\right) = 
 \tilde\pi\{\eta^\runter\}  \lambda(\eta^\runter \curvearrowright \gamma^\runter)
\]
using the balancing property of $g$. Else, if $\eta \not\in \Insert(\gamma)$, also $\gamma \notin \Delete(\eta)$ and $\tilde \pi\{\gamma^{\rauf}\}\lambda(\gamma^{\rauf} \curvearrowright \eta^{\rauf}) = \tilde\pi\{\eta^\runter\}  \lambda(\eta^\runter \curvearrowright \gamma^\runter)= 0.$\\
Finally, we obtain the semi-local condition (equation \eqref{semilocal}), 
$\Lambda( \gamma^\rauf) = \sum_{\B \in \States} \lambda ( \gamma^\rauf \curvearrowright \B)
= \sum_{\eta \in \Insert(\gamma)} \lambda(\gamma^\rauf \curvearrowright \eta^\rauf)
+ \lambda(\gamma^\rauf \curvearrowright \gamma^\runter)\\
=  \sum_{\eta \in \Delete(\gamma)}  \lambda(\gamma^\runter \curvearrowright \eta^\runter) =
\sum_{\A \in \States} \lambda (\gamma^\runter \curvearrowright \A) \\ = 
\Lambda(\invo(\gamma^\rauf)).$
Thus the theorem is proved.
\end{proof}

\section{GES as limit of our sampler}\label{sec:ges}

It is interesting to note that when starting in the empty graph with the balancing function $g(x) = \sqrt{x}$ and target 
$
\pi \{\gamma\} = \exp(\beta s(\gamma)),
$ where $\beta > 0$ is a coldness parameter and $s$ is a Markov equivalent score, we recover the greedy equivalence search algorithm (GES) in the limit $\beta \to \infty$. 
In this limit, the $\Insert$ operator that improves the score the most is selected immediately with probability approaching 1 as long as there is such  an edge addition which improves the score at all. 
This is because for $\eta \in \Insert(\gamma)$,
\[\kappa_{\gamma^\rauf} \{\eta^\rauf\} =  
\frac{\exp(\frac12\beta (s(\eta)-s(\gamma))}{\sum\limits_{\zeta \in \Insert(\gamma)} \exp(\frac12\beta(s(\zeta)-s(\gamma)))}
\]
is a soft-max over the score improvements and the intensity  $\Lambda(\gamma)$ approaches infinity. If no edge addition can improve the score anymore, the direction changes immediately if there is an edge removal that increases the score. In following second phase, again with probability approaching one, the $\Delete$ operator that improves the score the most is immediately selected with probability approaching 1 by same argument. This way the process reaches with probability approaching 1 in time approaching 0 the highest scoring model along the same trajectory as the GES with the same computational effort as a GES (when implemented with the same algorithmic improvements given below). This proves the following statement:

\begin{theorem}
If started in the empty graph, with balancing function $g(x) = \sqrt{x}$, for all $t > 0$,
\[
\lim_{\beta \to \infty} \P(Z_t \in \{\gamma_\star^\rauf, \gamma_\star^\runter\}) = 1, 
\]
where $\gamma_\star$ is the CPDAG found by a two-pass greedy equivalence search starting in the empty graph.  

Moreover, for large $\beta$, with high probability $Z$ visits the same models as the two-phase GES, with the same computational effort.
\end{theorem}
We refer to the thorough discussion in section 4 of \citet{chickering2002optimal}. In particular, we conclude with the remark  in section 4.3 that starting in the empty graph is an efficient way to converge towards the concentration of posterior mass in the large sample limit. Behaviour of piecewise deterministic processes under similar annealing schemes has been previously studied in \citet{Monmarch2016}.

\section{Efficient algorithms for the underlying graph operations}
\label{section:algorithms}
Before stating our algorithmic results, it is necessary to revisit a basic problem in this area: computing a DAG in the MEC represented by a given CPDAG. It is well-known that this task can be solved in linear-time $O(n+m)$ for CPDAGs with $n$ vertices and $m$ edges relying on algorithms from the chordal graph literature~\citep{chickering2002learning}.
The key observation is that the directed edges of the CPDAG can be ignored and any acyclic and v-structure-free orientation of the undirected edges, will yield a DAG from the MEC. 
This task can be performed using, e.g., by the graph traversal algorithm \emph{Maximum Cardinality Search}, MCS for short~\citep{tarjan1984simple}, which, at each step, visits a vertex with the highest number of already visited neighbours.  Appendix~A.2 in \citep{chickering2002learning} gives a good overview over this approach. 
More generally, the term \emph{consistent extension} is used to describe a DAG with the same adjacencies and v-structures as a given (C)PDAG. 

The computational task of applying one of the GES operators is fundamental, not only in the context of this work, but naturally also for GES itself and other score-based algorithms. 
Classically, the following approach is used, as described by~\citet{chickering2002learning}: First, the operator is applied locally by inserting/deleting the edge and orienting edges incident to $T$, respectively $H$, yielding a PDAG. 
Second, for this PDAG, a consistent extension is computed. Third, the new CPDAG is directly computed from the consistent extension. 

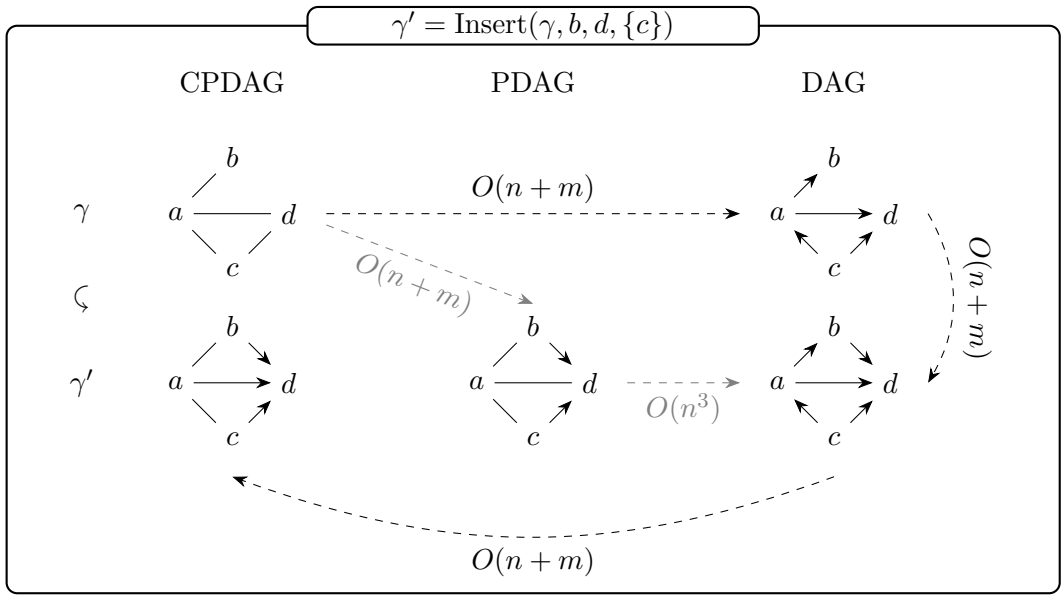
\begin{figure}[ht]
\centering
\begin{tikzpicture}[xscale=1.0]
    \draw[rounded corners, thick] (-3.0,-6.8) rectangle
      (11,0.75);
    \node[draw=black,rounded corners, thick, fill=white, inner
      sep=0pt, minimum height=0.5cm, baseline, minimum width=6cm] at
      (4,0.75) {$\gamma' = \Insert(\gamma, b, d, \{c\})$};
    \node (gam) at (-2, -1.75) {$\gamma$};
    \node[rotate=90] (g) at (-2,-2.875) {$\curvearrowleft$};
    \node (gamd) at (-2, -4) {$\gamma'$};

    \node (cpdag) at (0,0) {CPDAG};
    \node (pdag) at (4,0) {PDAG};
    \node (dag) at (8,0) {DAG};
    \node (a) at (-0.75,-1.75) {$a$};
    \node (b) at (0,-1) {$b$};
    \node (c) at (0,-2.5) {$c$};
    \node (d) at (0.75,-1.75) {$d$};
    \draw (b) -- (a) -- (d) -- (c) -- (a);
    \node (a) at (-0.75,-4) {$a$};
    \node (b) at (0,-3.25) {$b$};
    \node (c) at (0,-4.75) {$c$};
    \node (d) at (0.75,-4) {$d$};
    \draw (b) -- (a) -- (c);
    \draw[style=edge] (b) to (d);
    \draw[style=edge] (a) to (d);
    \draw[style=edge] (c) to (d);
    \node (a) at (3.25,-4) {$a$};
    \node (b) at (4,-3.25) {$b$};
    \node (c) at (4,-4.75) {$c$};
    \node (d) at (4.75,-4) {$d$};
    \draw (b) -- (a) -- (c);
    \draw (a) -- (d);
    \draw[style=edge] (b) to (d);
    \draw[style=edge] (c) to (d);
    \node (a) at (7.25,-1.75) {$a$};
    \node (b) at (8,-1) {$b$};
    \node (c) at (8,-2.5) {$c$};
    \node (d) at (8.75,-1.75) {$d$};
    \draw[style=edge] (c) to (d);
    \draw[style=edge] (c) to (a);
    \draw[style=edge] (a) to (d);
    \draw[style=edge] (a) to (b);
    \node (a) at (7.25,-4) {$a$};
    \node (b) at (8,-3.25) {$b$};
    \node (c) at (8,-4.75) {$c$};
    \node (d) at (8.75,-4) {$d$};
    \draw[style=edge] (c) to (d);
    \draw[style=edge] (c) to (a);
    \draw[style=edge] (a) to (d);
    \draw[style=edge] (a) to (b);
    \draw[style=edge] (b) to (d);
    
    \draw[dashed, style=edge] (1.25,-1.75) to node[midway,above] {$O(n+m)$} (6.75,-1.75);
    \draw[dashed, style=edge] (9.25,-1.75) to[bend left] node[midway,rotate=270,above]{$O(n+m)$} (9.25,-4);
    \draw[dashed, style=edge] (8,-5.25) to[bend left=20] node[midway,below]{$O(n+m)$} (0,-5.25);

    \draw[dashed, style=edge, color=gray] (1.25,-1.9) to node[midway,below,rotate=-21.25,xshift=-0.1cm] {$O(n+m)$} (4,-2.95);
    \draw[dashed, style=edge, color=gray] (5.25,-4) tonode[midway,below]{$O(n^3)$} (6.75,-4);
\end{tikzpicture}
\label{figure:applymove}
\caption{A schematic overview of the linear-time approach for applying a GES operator. Previous approaches add the inserted edge to the initial CPDAG, obtaining a PDAG associated with the new MEC $\gamma'$. However, going from this PDAG to the CPDAG, usually via a consistent DAG extension as intermediate step, necessitates time $O(n^3)$. In contrast, our approach finds a consistent DAG extension of the \emph{initial} CPDAG in time $O(n+m)$, which has the property that applying the operator \emph{directly} yields a DAG from $\gamma'$. Transforming this DAG into its CPDAG can be done in $O(n+m)$, as shown by~\citet{chickering1995transformational}.}
\end{figure}

The first and third step can be performed in linear-time, however, the second step, when performed naively, needs time $O(n^3)$~\citep{dor1992simple,wienobst2021extendability}. We provide a linear-time algorithm for this problem by modifying the first and second step, building on ideas from~\citet{chickering2002optimal} and~\citet{hauser2012characterization}:
\begin{theorem}\label{thm:extension}
    Let $\gamma$ be a CPDAG.
    Applying a GES operator $\on{Insert}(\gamma, x, y, T)$ or $\on{Delete}(\gamma, x, y, H)$ to $\gamma$ and obtaining $\gamma'$ is possible in time $O(n+m)$. 
 \end{theorem}
 \begin{proof}
    By Theorem~15 and 17 in \citep{chickering2002optimal}, any GES operator corresponds to a single edge insertion/deletion in a certain DAG in the MEC of $\gamma$. Our approach is as follows. First, compute a consistent extension of $\gamma$, which has the property that a single insertion/deletion yields a DAG from the new MEC represented by $\gamma'$ in linear-time. 
    Exploiting that $\gamma$ is a CPDAG allows us to find this consistent extension $G$ in linear-time using a modified MCS (described below). Then, the insertion/deletion can be performed in constant time to yield DAG $G'$. Afterwards, the ``standard'' third step of finding CPDAG $\gamma'$ for DAG $G'$ is applied~\citep{chickering1995transformational}.

    To perform the first step, we distinguish between the $\on{Insert}$ and $\on{Delete}$ operator. In case of the $\on{Insert}(\gamma, x, y, T)$, we perform an MCS which starts with visiting the vertices in $T$ and $\NA_x(y)$. As they form a clique, it is easy to see that this does not violate the properties of an MCS (the visit order is one which could be produced by a "standard" MCS). As discussed in the proof of Theorem~15 in \citep{chickering2002optimal} and Proposition~43 in~\citep{hauser2012characterization} , this yields a DAG $G$ with the desired property that inserting $x \rightarrow y$ gives $G' \in \gamma'$.
    For the $\mathrm{Delete}(\gamma, x, y, H)$ operator, we proceed the same way only that vertices in $\NA_x(y) \setminus H$ are visited first (afterwards $x$ and $y$ in this order). By the proof of Theorem~17 in \citep{chickering2002optimal}, this gives a DAG $G' \in \gamma'$.
 \end{proof}

This time-complexity is asymptotically optimal, as there are graphs, for which $O(m)$ edges change after applying an operator.

In the framework described above, to obtain a \emph{uniform} MCMC sampler of CPDAGs, it suffices to count the number of operators and to sample an operator with uniform probability. We derive the first polynomial-time algorithm for this task.\footnote{The proof is provided in Appendix~B in the supplement.} 

\begin{theorem}\label{theorem:countops}
    Let $\gamma$ be a CPDAG.  
    The number of locally valid Insert and Delete operators can be computed in time $O(n^2\cdot m)$. Sampling an operator uniformly is possible in the same time complexity.
\end{theorem}
Sampling an operator in polynomial-time in this manner is only possible in the uniform case. When operators are weighted by their score, a different procedure is necessary.

There are multiple possible approaches to sample an operator proportional to an underlying local score, which may update after a move. In this work, we rely on the fact that, per move, usually only a few operator scores change. Hence, we use (i) caching of local scores to only recompute scores, which actually change. This is, as in the GES algorithm, crucial as the score computation can be the bottleneck of the algorithm (depending on sample size and the particular scoring procedure). Then, we (ii) efficiently list all operators one-by-one (without generating invalid operators), enabled by the insights from the previous section. 

\begin{corollary}
    Let $\gamma$ be a CPDAG with 
    maximum number of neighbors $d$. The operators can be listed in time $O(n^2\cdot m + |\mathrm{op}(\gamma)|\cdot d)$.  
\end{corollary}

Using this result and caching, the overall cost per move is in $O(n^2\cdot m + |\mathrm{op}(\gamma)| \cdot d + |\mathrm{changed}(\gamma)| \cdot \mathrm{scoreeval})$, where $\mathrm{scoreeval}$ describes the time of a score evaluation. 
In our empirical studies, we find that the number of operators per pair of vertices is often constant (when the undirected edge degree is constant) and that the number of changed operators is usually very small, making the algorithmic improvements impactful. 

\section{Conclusions} 
We provide a novel continuous-time momentum-based MCMC sampler over the space of MECs based on the GES operators~\citep{chickering2002optimal} and extended by a notion of direction. We show empirically that it can achieve favourable mixing time compared to earlier MCMC approaches and apply an efficient implementation of this sampler to the problem of observational causal discovery. In particular, our algorithmic improvements regarding the application of the GES operators, yielding linear-time for applying an operator and polynomial-time for counting the number of operators, go beyond this specific use case.

\newpage

\bibliography{ref}

\newpage

\section*{Acknowledgements}

Data used in preparation of this article were obtained from the Alzheimer’s Disease Neuroimaging Initiative (ADNI) database (\url{adni.loni.usc.edu}).

As such, the investigators within the ADNI contributed to the design and implementation of ADNI and/or provided data but did not participate in the writing of this article. A complete listing of ADNI investigators can be found at: \url{http://adni.loni.usc.edu/ wp-content/uploads/how_to_apply/ADNI_Acknowledgement_List.pdf}.

ADNI data is de-identified and publicly available for download.
All study participants provided written informed consent, and study protocols were approved by
each local site’s institutional review board.

{\small
Data collection and sharing for the ADNI project was funded by the
Alzheimer's Disease Neuroimaging Initiative (ADNI) (National
Institutes of Health Grant U01 AG024904) and DOD ADNI (Department of
Defense award number W81XWH-12-2-0012). ADNI is funded by the National
Institute on Aging, the National Institute of Biomedical Imaging and
Bioengineering, and through generous contributions from the following:
AbbVie, Alzheimer’s Association; Alzheimer’s Drug Discovery
Foundation; Araclon Biotech; BioClinica, Inc.; Biogen; Bristol-Myers
Squibb Company; CereSpir, Inc.; Cogstate; Eisai Inc.; Elan
Pharmaceuticals, Inc.; Eli Lilly and Company; EuroImmun; F.
Hoffmann-La Roche Ltd and its affiliated company Genentech, Inc.;
Fujirebio; GE Healthcare; IXICO Ltd.; Janssen Alzheimer Immunotherapy
Research \& Development, LLC.; Johnson \& Johnson Pharmaceutical
Research \& Development LLC.; Lumosity; Lundbeck; Merck \& Co., Inc.;
Meso Scale Diagnostics, LLC.; NeuroRx Research; Neurotrack
Technologies; Novartis Pharmaceuticals Corporation; Pfizer Inc.;
Piramal Imaging; Servier; Takeda Pharmaceutical Company; and
Transition Therapeutics. The Canadian Institutes of Health Research is
providing funds to support ADNI clinical sites in Canada. Private
sector contributions are facilitated by the Foundation for the
National Institutes of Health (\url{www.fnih.org}). The grantee
organisation is the Northern California Institute for Research and
Education, and the study is coordinated by the Alzheimer’s Therapeutic
Research Institute at the University of Southern California. ADNI data
are disseminated by the Laboratory for Neuro Imaging at the University
of Southern California.}

\newpage

\appendix

\section{Skew-balanced jump processes} 

\begin{prop}\label{prop:stationary}
If there is an bijection $\invo$ on $S$ that is $\pi$-isometric:
\begin{equation}\label{isometry}
\pi\{a\} = \pi\{\invo(a)\}, \quad a \in \States,
\end{equation}
such that \emph{skew detailed balance} 
\begin{equation}\label{skewbalance}
\pi\{\A\}\lambda(\A \curvearrowright \B) = \pi\{\invo(\B)\} \lambda(\invo(\B) \curvearrowright \invo(\A)) \quad \A, \B \in \States
\end{equation}
holds and 
such that the \emph{semi-local} condition 
\begin{equation}\label{semilocal}
\Lambda(\A) = \Lambda(\invo(\A)) 
\end{equation}
holds, then $Z$ is $\pi$-stationary.
\end{prop}

\eqref{semilocal} typically requires that $\invo^n$ for some order $n = 1, 2, \dots$ is the identity map. If $\invo$ is the identity ($n=1$), then \eqref{semilocal} and \eqref{isometry} hold automatically and 
\eqref{skewbalance} reduces to a detailed balance condition.

Also the case $n = 2$ is important. A map $\invo\colon \States \to \States$ is an \emph{involution} if $\invo\circ \invo$ is the identity. 
For example, if $\States = \cX \times \{-1,1\}$, then $\invo$ with $\invo((x, d)) = \invo((x, -d))$ for $(x, d) \in \States$ is an involution. An involution is automatically an bijection.
Importantly, \eqref{skewbalance} is trivial for $\B = \invo(\A)$, but turns into a linear constraint if designing samplers using $\invo$ with higher orders $n$. 

\section{Remaining proofs}

A convenient criterium for stationary is as follows:
If $Z$ is stationary for $\pi$, then for bounded  $f\colon \States \to \RR$
\begin{equation}\label{timelocalstationary}
\sum_a \sum_b \lambda(a\curvearrowright \B) (f(\B) - f(a))  \pi\{a\} = 0.
\end{equation}
Conversely, if the preceding equation holds for all  $f\colon \States \to \RR$ bounded, 
then $Z$ is stationary.
]

\begin{proof}[Proof of proposition \ref{prop:stationary}]

The proposition follows from \eqref{skewbalance} by
$\sum_\A \sum_\B \lambda(\A\curvearrowright \B) f(\B)  \pi\{\A\} =$ $\sum_\A \sum_\B  \lambda(\invo(\B) \curvearrowright \invo(\A)) f(\B) \pi\{\invo(\B)\}$
and with $z = \invo(\A)$,
\[
= \sum_\B  \sum_z \lambda(\invo(\B) \curvearrowright z) f(\B) \pi\{\invo(\B)\} 
\]
and by the definition of the specific rate and its connection to the total
\[
=\sum_\B  \Lambda(\invo(\B)) f(\B)\pi\{\invo(\B)\}  = \sum_\A \sum_\B \lambda(\B\curvearrowright \A) f(\B)  \pi\{\B\} 
\]
\[
=  \sum_\A \sum_\B \lambda(\A\curvearrowright \B) f(\A)  \pi\{\A\} . 
\]
In the last step we use that $\Lambda(\invo(\B)) =\Lambda(\B) = \sum_{\A} \lambda(\B \curvearrowright \A)$ and $\pi\{\invo(\B)\} = \pi\{\B\}$.
So we have established \eqref{timelocalstationary} for any summable $f$.
\end{proof}

\begin{proof}[Supplement to the proof of Theorem \ref{irreducible}.]
  Let $\gamma^d \in \States$, where $\gamma$ is not the  graph with no edges $\mathbf 0_n$ (assume that $n > 1$ so there is something to show.) We now prove $\P(Z_t = \mathbf 0_n^\runter \mid Z_s = \gamma^d) > 0, \gamma^d \in \States$, $t > s$. 
We first find a state $\eta^\runter$ such that $\P(Z_{(t-s)/2} = \mathbf \eta^\runter \mid Z_s = \gamma^d) > 0, \gamma^d \in \States$. If $d = \runter$, one can take  $\eta = \gamma$. Otherwise, if $d = \rauf$,
though $\Delete(\gamma)$ is non-empty\footnote{$\gamma$ has edges, so there is a DAG $G \in \gamma$ from which an edge can be removed to obtain some $G' \in \Delete(\gamma)$}, it can still be that $\lambda(\gamma^\rauf \curvearrowright \gamma^\runter) =0$. But in that case, by construction, $\Insert(\gamma)$ is non-empty and $\lambda(\gamma^\rauf \curvearrowright \zeta^\rauf) > 0$ for some $\zeta \in \cM_n$. Repeating that argument, at most $n(n-1)/2$ many jumps lead to a state $\eta^\runter \in \States$ and together, these jumps have positive probability to occur in a time interval of length $(t-s)/2$, so $\eta^\runter$ is that state we are looking for.

Now from $\eta$ there is a sequence of at most $n(n-1)/2$  edge removal moves that reach $\mathbf 0_n$. Together these jumps have again positive probability to occur in a time interval of length $(t-s)/2$. Therefore $\P(Z_t = \mathbf 0_n^\runter \mid Z_{s} = \gamma^d) \ge  \P(Z_{t} = \mathbf {\mathbf 0}_n^\runter \mid Z_{(t-s)/2} = \eta^\runter) \P(Z_{(t-s)/2} = \mathbf \eta^\runter \mid Z_s = \gamma^d) > 0$. 

By repeating this argument,  
$\P(Z_t ={\mathbf 0}_n^\runter \mid Z_{(t-s)/2} = \gamma^{-d}) > 0$. Using skew balance to reverse the path from $\gamma^{-d}$ to ${\mathbf 0}_n^\runter$ into a path from ${\mathbf 0}_n^\rauf$ to $\gamma^{d}$, replacing edge inserts by edge deletions and vice versa,
$\P(Z_t =\gamma^d \mid Z_{(t-s)/2} = {\mathbf 0}_n^ \rauf)>0$.

Also $\P(Z_{(t-s)/2} ={\mathbf 0}_n^\rauf \mid Z_s = {\mathbf 0}_n^ \runter) > 0$ as $\lambda({\mathbf 0}_n^ \runter \curvearrowright {\mathbf 0}_n^ \rauf) > 0$
because there is no delete operator available, but one can insert an undirected edge to $\mathbf 0_n$. 
We therefore have $\P(Z_t =\gamma^d \mid Z_s = {\mathbf 0}_n^ \runter) \ge 
\P(Z_t =\gamma^d \mid Z_{(t-s)/2} = {\mathbf 0}_n^ \rauf) \P(Z_{(t-s)/2} ={\mathbf 0}_n^\rauf \mid Z_s = {\mathbf 0}_n^ \runter) > 0$.

This is sufficient because $S$ is finite.
\end{proof}

\begin{proof}[Proof of Theorem~\ref{theorem:countops}]
    The task immediately reduces to counting the number of operators for each pair of vertices. We consider the $\mathrm{Delete}(\gamma, x, y, H)$ operator first. Here, the set of operators correspond to the subsets of $\NA_x(y)$, which form a clique. This further reduces to the problem of counting (and sampling) the number of cliques of a chordal graph, that is a graph without induced cycles of length $\geq 4$,~\citep{dirac1961rigid} due to the fact that there can only be undirected edges between vertices in $\NA_x(y)$ (Lemma~3 \citep{chickering1995transformational}) and that these undirected edges form a chordal graph in a CPDAG \citep{andersson1997characterization}. It is a basic fact that the number of cliques of a chordal graph $\gamma$ is given by:
    \[
        \prod_{u \in V} 2^{\Pa_D(u)} + 1,
    \]
    where $G$ is any consistent extension of $\gamma$ due to the fact that all parents of $u$ form a clique (else $D$ would not be a consistent extension as it has additional v-structures). Each term in the product gives the number of cliques containing $u$ as highest ordered vertex w.r.t. some fixed topological ordering of $D$.  Evaluating this is clearly possible in $O(m)$ per pair $x, y$.

    For the $\mathrm{Insert}(\gamma, x, y, T)$ operator, the set of operators is formed by subsets $T$ of the undirected neighbours of $y$, which are nonadjacent with $x$, such that $\NA_x(y) \cup T$ is a clique and $\NA_x(y) \cup T$ blocks all paths from $y$ to $x$ without edge pointing towards $y$. 
    The latter condition complicates the matter. It can be resolved as follows: Consider, for each neighbor of $y$, the set of vertices reachable via a path without edges pointing towards $y$ (that is reachable via a semi-directed path $y - x \dots$)  not containing an undirected neighbour of $y$. This can be done independently of $x$ taking overall time (for all $y$) $O(n^2 m)$. If, under these constraints, $x$ is reachable from a neighbour $w$ of $y$, which is non-adjacent to $x$, then $w$ has to be in $T$ (else there is an open semi-directed path from $y$ to $x$). After taking all such vertices $w$, none of the remaining vertices has an open semi-directed path to $x$. We show this by contradiction. Assume there exists $z$ such that there is a semi-directed path from $z$ to $x$ not blocked by $\NA_x(y) \cup T$. There has to be a vertex $a$ on this path, which is a neighbour of $y$ else $z$ would be in $T$, consider the one closest to $x$. Then, this vertex has an unblocked semi-directed path to $x$ and hence is in $T$. This is a contradiction to the fact that the path is open given $\NA_x(y) \cup T$. 
    
    Hence, we can compute the set of vertices, which \emph{must} be in $T$ in overall time $O(n^2 m)$, respectively $O(m)$ per pair $x,y$. Consequently, they need to form a clique with $\NA_x(y)$ (this can be checked in $O(m)$ as well). The remaining neighbours of $y$ (non-adjacent with $x$), which are fully connected to $\NA_x(y)$ and the must-take vertices, may be part of $T$ as long as they themselves form a clique. Hence, we arrive at the problem of counting the number of cliques in a chordal graph studied above, which can be solved in time $O(m)$. 

    It is easy to see that sampling can be performed in time $O(n^2 m)$ (when performing counting as preprocessing) by first sampling a pair of vertices $x,y$ with probability proportional to the number of locally valid operators and second sampling an operator for this set with uniform probability (which amounts to sampling a clique in a chordal graph). 
\end{proof}

\section{Further models with high posterior probability for the ADNI database}
In the main document, we only showed the DAG with highest posterior probability of 0.701 for the data from the ADNI database.
In figure~\ref{fig:adniappendix}, we show the DAGs with the second- and third-highest posterior probabilities, which are 0.207 and 0.0049.  

\begin{figure*}[hb!]
\centering

\includegraphics[width=0.45\linewidth]{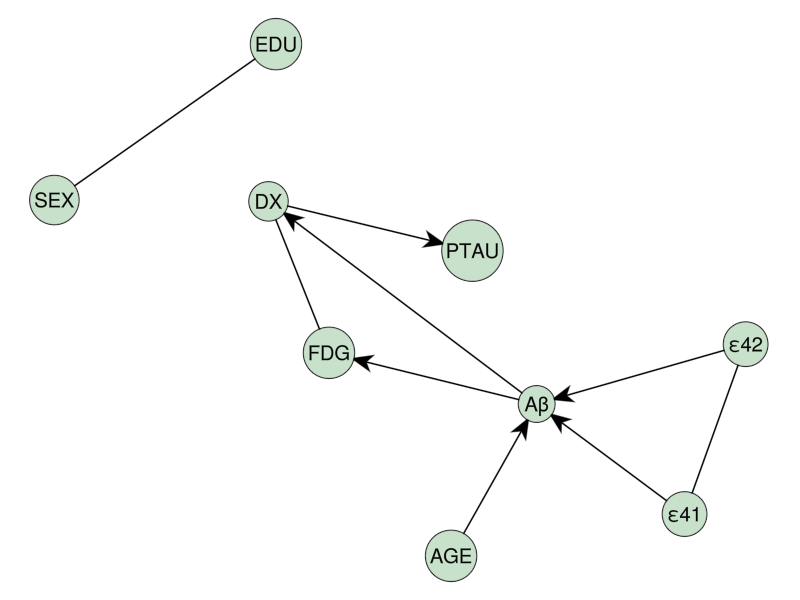}
\hspace{0.5cm}
\includegraphics[width=0.45\linewidth]{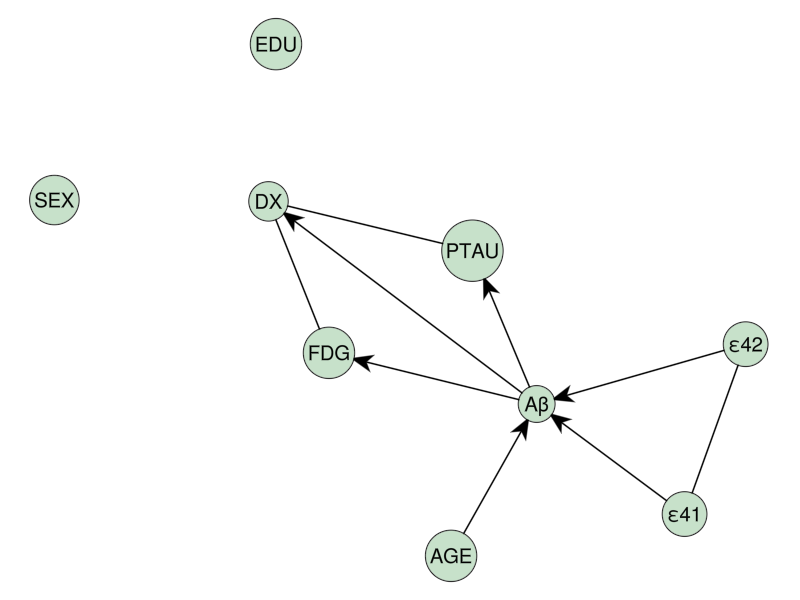}
\caption{The two models with second- and third-highest posterior probabilities, namely 0.207 and 0.0049. This illustrates one particular use case of our sampler, namely uncertainty quantification in the situation where GES applies, such as uncertainty about the edge SEX to EDU. }
\label{fig:adniappendix}
\end{figure*}

\end{document}